%% file: main.tex
\documentclass[letterpaper]{article}
\usepackage[preprint]{aaai2027}

\usepackage[hyphens]{url}
\usepackage{graphicx}
\usepackage{natbib}
\usepackage{caption}
\usepackage{amsmath}
\usepackage{amssymb}
\usepackage{array}
\usepackage{booktabs}
\usepackage{multirow}
\newtheorem{proposition}{Proposition}
\newtheorem{corollary}{Corollary}

\usepackage{algorithm}
\usepackage{algorithmic}

\title{HeadWiseKV: Budgeted Per-Head Cache Residency for Hybrid Long-Context Language Models}

\author{
    Renjie Xie\textsuperscript{\rm 1,\rm 2}\equalcontrib,
    Juncheng Yang\textsuperscript{\rm 2}\equalcontrib,
    Aoting Hu\textsuperscript{\rm 3,\rm 2}\equalcontrib,\\
    Mingxi Zhang\textsuperscript{\rm 1,\rm 2},
    Liyao Wu\textsuperscript{\rm 1,\rm 2},
    Zheheng Hong\textsuperscript{\rm 4,\rm 2},
    Wei Xu\textsuperscript{\rm 5}
}
\affiliations{
    \textsuperscript{\rm 1}Nanjing University of Posts and Telecommunications\\
    \textsuperscript{\rm 2}Tylogi AI Lab / TAIL\\
    \textsuperscript{\rm 3}Anhui University of Technology\\
    \textsuperscript{\rm 4}Shanghai Jiao Tong University\\
    \textsuperscript{\rm 5}School of Information Science and Engineering, Southeast University\\
    renjie\_xie@njupt.edu.cn, juncheng\_yang@tylogi.com,
    aotinghu@ahut.edu.cn, b23080111@njupt.edu.cn,
    lywu@tylogi.com, zh2227@sjtu.edu.cn, wxu@seu.edu.cn
}

\begin{document}

\maketitle

\begin{abstract}
Long-context inference retains a growing key--value (KV)
cache during decoding, which consumes substantial GPU memory and can reduce
generation throughput. This bottleneck remains in hybrid language
models because their residual global-attention layers can dominate
context-dependent cache demand. We study how to allocate
this state under an aggregate KV-residency budget. We introduce HeadWiseKV, a
training-free framework that compresses the residual global KV
caches of hybrid language models while preserving their native local,
recurrent, and linear paths. It assigns each physical KV head a static,
multilevel history window, making cache demand predictable before serving. We
formulate this allocation as a restricted operational rate--distortion problem
and propose SeqCalib as the core policy-generation algorithm in HeadWiseKV.
SeqCalib processes layers in execution order and conditions each decision on
the lower-layer policy used at deployment, thereby accounting for interactions
across depth. A grouped-cache runtime materializes the selected policy as
actual per-head KV residency rather than a mask over a full cache. We evaluate
downstream quality across four hybrid long-context models and study physical
residency and serving behavior on Qwen3.6-27B. HeadWiseKV retains near-Full-KV
RULER and LoCoMo quality across the evaluated models. In the fixed-model
systems study, it reduces sampled peak device memory by 8.59\% at a 112K
context length and extends the largest verified successful context from 114K
to 161K.
\end{abstract}

\input{sections/introduction}
\input{sections/related_work}
\input{sections/method}
\input{sections/evaluation}
\input{sections/conclusion}

\bibliography{references}

% The AAAI supplementary document is merged into the public preprint so that
% all proofs, implementation details, additional results, and limitations are
% available in one PDF and use the same cross-reference namespace.
\clearpage
\appendix
\input{sections/theory_appendix}
\input{sections/evaluation_details}
\input{sections/qwen35_9b_throughput_figure}
\input{sections/runtime_implementation}
\input{sections/longmemeval_results}
\input{sections/limitations}

\end{document}

%% file: sections/introduction.tex
\section{Introduction}

During long-context autoregressive inference, each full-attention layer stores
key and value vectors for every cached token. This KV cache grows with context
length and concurrency and can exhaust device memory \citep{h2o}. Hybrid
language models reduce the burden but do not eliminate it.
Gemma~2 and Gemma~3 interleave global and local sliding-window attention
\citep{gemma2,gemma3}, whereas Qwen3-Next and Qwen3.6 interleave full or gated
attention with recurrent Gated-DeltaNet blocks
\citep{qwen3next_blog,qwen36_27b,gated_deltanet}. Their remaining global blocks
still retain histories that grow with the prompt. Hybridization therefore
relocates the KV bottleneck rather than eliminating it: a small set of residual
global layers can determine whether a long request fits in memory.

KV compression involves two distinct choices: when the
retention policy is determined and how that policy is realized in memory.
Prompt-dependent methods derive token or head decisions from the current
request, observed attention, or online eviction state
\citep{fastgen,adakv,headkv,kvcompress,diffkv,cake}. This adaptation can
preserve isolated distant evidence that is important to a particular prompt.
Its systems benefit, however, depends on the runtime. A logical sparse mask can
reduce attention work while leaving the allocated KV cache unchanged.
Compaction or eviction can reduce decoding-time residency, but a method that
first constructs a full cache need not reduce the prefill memory peak. Online
scoring, indexing, and KV movement also add request-time overhead, so sparse
attention does not automatically imply higher end-to-end throughput.

Static policies make the complementary trade-off.
Capacities fixed before prefill allow a runtime to allocate bounded storage
directly, yielding predictable memory demand without request-time selection.
The policy cannot inspect the current prompt,and therefore cannot
identify which isolated past tokens a future query may need. Existing static
methods mitigate this limitation by assigning full history to retrieval heads
and compressed or streaming caches to other heads
\citep{razorattention,duoattention}. This binary split is regular to deploy but
can be too coarse for hybrid models whose residual global heads require
different amounts of history. Table~\ref{tab:closest-work} separates policy
timing from physical cache realization and summarizes this trade-off.

Static allocation also creates a calibration problem across
depth. Sensitivity varies by head and layer
\citep{fastgen,razorattention,headkv,diffkv,squeezeattention,pyramidkv,d2o,cake}.
A common window can waste memory on tolerant heads while truncating sensitive
ones. Moreover, shortening an early cache changes the representations consumed
by later layers. A higher layer calibrated under an all-full prefix may
therefore see different inputs after deployment. A deployable static policy
must choose fine-grained per-head capacities under the lower-layer decisions
that will actually be active, and the runtime must realize those capacities as
physical storage.

We introduce \emph{HeadWiseKV} to meet these requirements.
It assigns each residual global KV head one of several contiguous history
lengths offline and keeps the resulting policy fixed during serving. This
multilevel allocation is more expressive than binary head specialization while
retaining a regular, prompt-independent layout. Its runtime stores only the
selected histories, so the policy determines physical residency rather than
logical access to an all-full cache. HeadWiseKV thus exchanges request-specific
token selection for a fine-grained memory plan that is known before prefill.

HeadWiseKV connects three technical pieces in an offline-to-online pipeline.
A structured residency model first defines the candidate history lengths and
their storage cost. At the core of HeadWiseKV is \emph{SeqCalib}, the offline
algorithm that selects one history length for every configurable KV head. For
each layer, SeqCalib keeps the previously selected lower-layer windows active
and compares candidate pre-gate attention outputs with a conditional
full-history reference under that same lower-layer policy. It chooses the
lowest-cost codebook entry meeting a mean-cosine threshold. For the realized
lower prefix, this finite search is stage-wise exact, not a guarantee of joint
global optimality. Running SeqCalib over a finite set of thresholds produces
static policy matrices spanning different residency costs. A deterministic
budget selector then chooses among the feasible matrices. A grouped cache
runtime loads the chosen matrix once and allocates the corresponding physical
history for each head. In short, the residency model defines the choices,
SeqCalib decides what to retain, and the runtime makes that decision real.
None of these steps requires retraining or architectural changes.
We cast the finite allocation as a \emph{restricted operational rate distortion
problem} \citep{shannon_fidelity,cover_thomas}. Its stage-wise guarantee does
not establish globally optimal joint allocation or downstream quality.

\begin{figure}[t]
    \centering
    \includegraphics[width=\columnwidth]{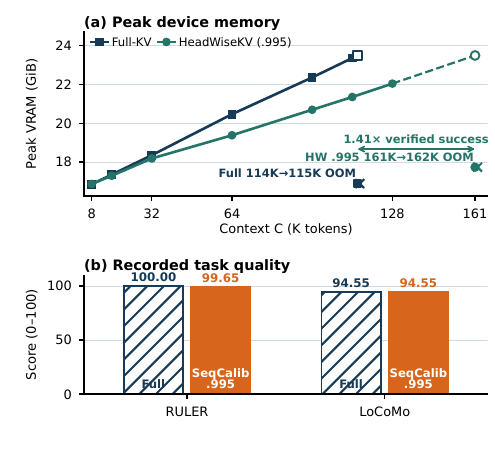}
    \caption{{\textbf{Qwen3.6-27B memory, capacity, and quality summary.}
    HeadWiseKV lowers sampled peak device memory and expands the verified
    context range while preserving long-context quality. Panel (a) combines
    matched memory measurements with
    independent adjacent-grid capacity brackets; panel (b) reports recorded
    RULER and LoCoMo scores. The systems and quality panels use separately
    recorded \(\tau=.995\) operating points and do not form a joint
    quality--systems result.}}
    \label{fig:qwen27-summary}
\end{figure}

\newcommand{\placeclosestworktable}{%
\begin{table}[t]
\centering
\scriptsize
\setlength{\tabcolsep}{0.8pt}
\renewcommand{\arraystretch}{1}
\caption{\textbf{Mechanistic comparison of KV-cache methods.}
We compare Ada-KV + SnapKV \citep{adakv,snapkv}, StreamingLLM
\citep{streamingllm}, DuoAttention \citep{duoattention}, HeadKV
\citep{headkv}, RazorAttention \citep{razorattention}, and KV-Compress
\citep{kvcompress}. ``Offline capacities'' means that cache lengths are fixed
before requests. ``Prompt-independent retention'' means
that retention does not use request content or observed attention.
``Physical before prefill'' means
that the runtime reduces allocated storage before constructing the cache.
\(\checkmark\) denotes yes, \(\circ\) denotes partial, and \(\times\) denotes
no, --: not documented.}
\label{tab:closest-work}
\resizebox{\columnwidth}{!}{%
\begin{tabular}{@{}lccccc@{}}
\toprule
& \multicolumn{2}{c}{Policy}
& \multicolumn{2}{c}{Allocation}
& \multicolumn{1}{c}{Runtime} \\
\cmidrule(lr){2-3}\cmidrule(lr){4-5}\cmidrule(l){6-6}
Method
& \shortstack{Offline\\cache\\capacities}
& \shortstack{Prompt-\\independent\\retention}
& \shortstack{Multilevel\\per\\head}
& \shortstack{Deployment-\\conditioned\\calib.}
& \shortstack{Physical\\before\\prefill} \\
\midrule
Ada-KV + SnapKV      & \(\times\)     & \(\times\) & \(\checkmark\) & \(\times\) & \(\times\) \\
StreamingLLM & \(\checkmark\) & \(\checkmark\) & \(\times\) & \(\times\) & \(\checkmark\) \\
DuoAttention         & \(\checkmark\) & \(\checkmark\) & \(\times\)     & \(\times\) & \(\checkmark\) \\
HeadKV               & \(\checkmark\) & \(\times\) & \(\checkmark\) & \(\times\) & -- \\
RazorAttention       & \(\checkmark\) & \(\checkmark\) & \(\circ\)      & \(\times\) & \(\checkmark\) \\
KV-Compress          & \(\times\)     & \(\times\) & \(\checkmark\) & \(\times\) & \(\times\) \\
\midrule
\textbf{HeadWiseKV} & \(\boldsymbol{\checkmark}\) & \(\boldsymbol{\checkmark}\) & \(\boldsymbol{\checkmark}\) & \(\boldsymbol{\checkmark}\) & \(\boldsymbol{\checkmark}\) \\
\bottomrule
\end{tabular}}
\end{table}
}

{Figure~\ref{fig:qwen27-summary} summarizes the central
advantage of HeadWiseKV: it lowers device-memory demand and supports longer
contexts while preserving long-context quality.  On Qwen3.6-27B
\citep{qwen36_27b}, HeadWiseKV reduces sampled peak device memory by 8.59\% at
112K tokens and increases the largest verified successful context from 114K
to 161K, while retaining near-Full-KV performance on RULER and LoCoMo.  The
systems and quality panels report separately recorded operating points; full
protocols and results appear in Section~\ref{sec:evaluation}.}

\noindent\begin{minipage}{\columnwidth}
Our contributions are:
\begin{itemize}
    \item {a budgeted physical-residency formulation for the
    residual global KV heads of hybrid models, together with an analysis
    motivating nonuniform allocation and deployment-conditioned calibration;}
    \item {HeadWiseKV and its prefix-conditioned SeqCalib
    algorithm, which construct static, multilevel per-head suffix policies
    under a prescribed KV-residency budget; and}
    \item {a grouped physical per-head cache runtime that
    materializes the selected policy as actual cache residency, together with
    an evaluation across model families, compression budgets, and long-context
    workloads.}
\end{itemize}
\end{minipage}

%% file: sections/related_work.tex
\section{Related Work}

%\paragraph{KV retention policies.}
Prior work varies along two independent axes. Policy timing
determines whether retention is fixed offline or selected from the current
request. Cache realization determines whether that choice changes physical
residency or only sparse attention access. Request-dependent token selection
uses attention, position, or prompt-specific signals
\citep{scissorhands,h2o,snapkv,sagekv}. Head- and layer-aware
methods further vary budgets across the model
\citep{fastgen,razorattention,adakv,headkv,duoattention,kvcompress,diffkv,
squeezeattention,pyramidkv,d2o,cake}. A closely related direction reduces the
context consulted by each query. QUEST selects KV pages using query-dependent
bounds, while TokenSelect performs dynamic token-level selection across heads
\citep{quest,tokenselect}. InfLLM retrieves relevant blocks from an auxiliary
context memory, and PyramidInfer reduces retained tokens according to
layer-wise attention consistency \citep{infllm,pyramidinfer}. Together, these
methods adapt token or block access to the input or query. Their physical
memory and latency effects still depend on whether the runtime masks,
retrieves, compacts, or evicts the selected states.

In contrast, static methods move retention decisions before
the request. StreamingLLM fixes a sink--recent policy, while RazorAttention
and DuoAttention use offline head roles that primarily expose a binary choice
between full and compressed histories
\citep{streamingllm,razorattention,duoattention}. HeadWiseKV instead fixes a
multilevel suffix length for each physical KV head, calibrates higher layers
under the selected lower-layer policy, and allocates the resulting capacities
before prefill. Table~\ref{tab:closest-work} compares these mechanisms. The
``prompt-independent retention'' and ``physical before prefill'' columns
deliberately separate policy dependence from storage realization. The
``deployment-conditioned calibration'' column captures whether higher-layer
decisions are made under the lower-layer policy used at deployment. A
checkmark denotes a documented property, not an assumption that the property
is universally preferable.

This retention policy is orthogonal to cache precision and
serving-time memory management, which reduce representation cost or manage
allocation without choosing a history length for each head
\citep{kivi,kvquant,pagedattention,vattention}. HeadWiseKV specifically
targets the residual global-attention KV histories of hybrid models and leaves
their bounded-memory local or recurrent paths unchanged.

% Queue the two-column method overview before the Method section so that it can
% occupy the next page top rather than drift behind the section it introduces.

\begin{figure*}[t]
    \centering
    \includegraphics[width=\textwidth]{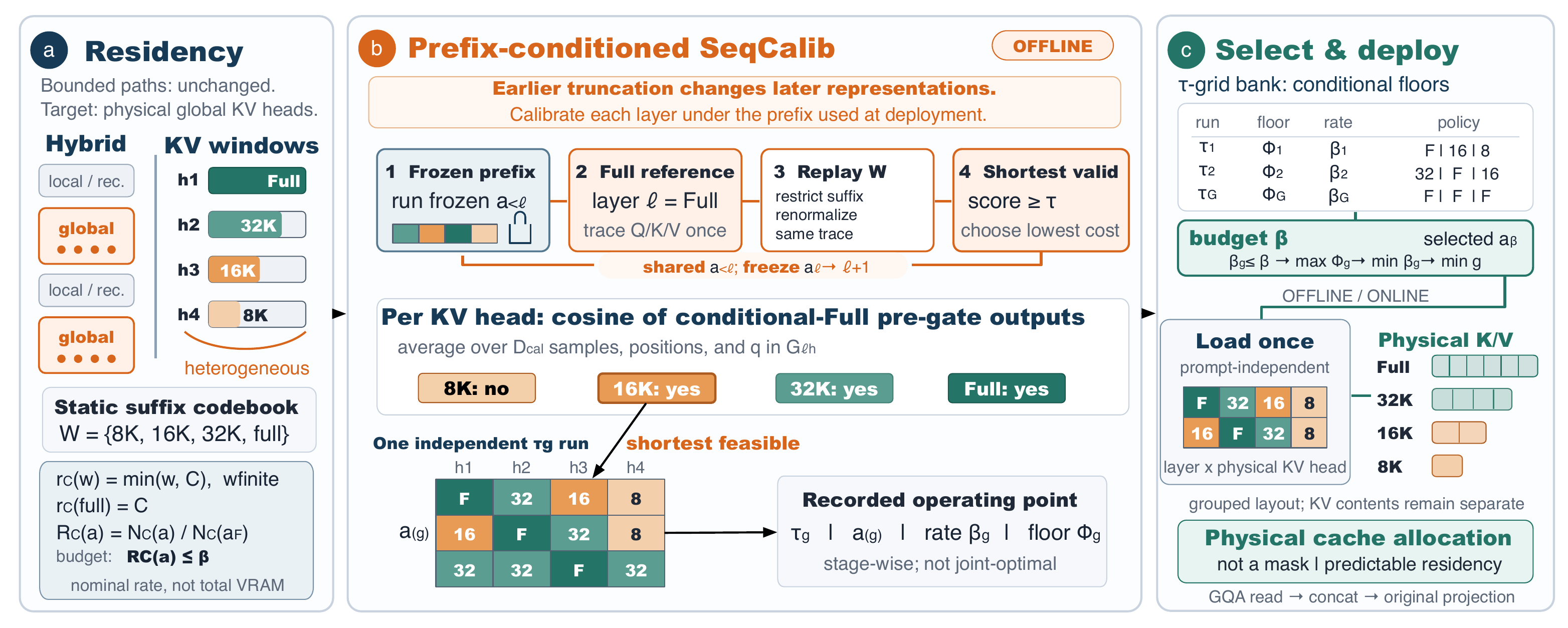}
    \caption{\textbf{Overview of HeadWiseKV.} Residual global-attention layers can have
    heterogeneous history requirements across KV heads. Offline SeqCalib
    evaluates a finite suffix codebook one layer at a time under the frozen
    lower-layer policy and produces a static policy matrix. The runtime loads
    this matrix once and materializes grouped per-head caches with different
    physical lengths. Window values are illustrative.}
    \label{fig:headwisekv-overview}
\end{figure*}

%% file: sections/method.tex
\section{Method}
\label{sec:method}

Figure~\ref{fig:headwisekv-overview} summarizes the offline-to-online pipeline.
HeadWiseKV defines feasible per-head history lengths and their storage costs,
uses SeqCalib to select a layer-conditioned policy, and realizes that policy as
physical cache residency. The next subsections formalize these stages.

\subsection{Structured Residency Model}

We begin by formalizing the static per-head residency problem.
At a fixed context length, the policy decides how much recent history each
configurable KV head retains. Only residual full-attention layers participate:
\(\mathcal L_{\mathrm F}\) lists these layers in forward execution order, and
\(\mathcal H_\ell\) contains the physical KV heads in layer \(\ell\). Under
grouped-query attention (GQA), several query heads may share one KV head. We denote this group by
\(\mathcal G_{\ell h}\). The decision unit is therefore a physical KV head
\((\ell,h)\), not an individual query head. The set of all decision units is
\(\mathcal U=\{(\ell,h):\ell\in\mathcal L_{\mathrm F},
h\in\mathcal H_\ell\}\), and HeadWiseKV assigns
\begin{equation}
 a_{\ell h}\in\mathcal W=\{w_1,\ldots,w_K,\mathrm{full}\},
\end{equation}
where each finite \(w_k\) denotes a contiguous suffix length. A single allocation
is shared by all query heads in \(\mathcal G_{\ell h}\) because they read the
same stored keys and values. For example, with
\(\mathcal W=\{8\mathrm K,16\mathrm K,32\mathrm K,\mathrm{full}\}\),
\(a_{\ell h}=8\mathrm K\) instructs KV head \(h\) in layer \(\ell\) to retain
its latest 8K KV states, while another head may retain 32K or full history. The
complete policy is thus a layer-by-KV-head matrix of history lengths. Direct
calibration assumes \(0<w_1<\cdots<w_K<C\). Any \(w\ge C\) is equivalent to
\(\mathrm{full}\). The order \(\prec_{\mathcal W}\), increasing finite windows
followed by \(\mathrm{full}\), resolves cost ties.

\placeclosestworktable

Storage grows linearly with the retained history. At context length
\(C\in\mathbb N_{>0}\), \(r_C(w)=\min(w,C)\) gives the retained token count for
a finite window, while \(r_C(\mathrm{full})=C\). Thus, at \(C=32\mathrm K\),
an 8K assignment stores one quarter as many KV elements as a full assignment
for the same head. The stored key/value dimensions
\(d^K_{\ell h},d^V_{\ell h}\) and effective precisions
\(b^K_{p_{\mathrm{KV}}},b^V_{p_{\mathrm{KV}}}\) convert this count into nominal
storage under cache format \(p_{\mathrm{KV}}\):
\begin{align}
 n_{\ell h,C}(w;p_{\mathrm{KV}})
   &=r_C(w)(d^K_{\ell h}b^K_{p_{\mathrm{KV}}}
                    +d^V_{\ell h}b^V_{p_{\mathrm{KV}}}),\\
 N_C(a;p_{\mathrm{KV}})&=\sum_{(\ell,h)\in\mathcal U}
             n_{\ell h,C}(a_{\ell h};p_{\mathrm{KV}}).
 \label{eq:nominal-storage}
\end{align}
Summing over \(\mathcal U\) gives the cost of the complete policy. When
\(p_{\mathrm{KV}}\) changes numerical execution, the full-cache reference,
compressed policy, and calibration scores use the same format. When it is used
only for accounting, it affects only the cost above. We normalize this total
by the cost of the all-full policy \(a^{\mathrm F}\):
\begin{equation}
 R_C(a;p_{\mathrm{KV}})
 =\frac{N_C(a;p_{\mathrm{KV}})}
        {N_C(a^{\mathrm F};p_{\mathrm{KV}})}.
 \label{eq:normalized-rate}
\end{equation}
A value \(R_C(a;p_{\mathrm{KV}})=0.5\), for example, means that the configurable
global KV heads require half the nominal storage of their all-full counterparts.
The normalized rate excludes model weights, workspaces, temporary buffers, and
allocator overhead. Section~\ref{sec:runtime} separately measures whole-device
video random-access memory (VRAM).

Let \(\mathcal X_C\) denote the input space at context length \(C\), and let
\(\mathcal Y_C\) denote the end-to-end output space on which fidelity is
measured. For any static allocation \(a\in\mathcal W^{|\mathcal U|}\), define
\(F_{a,C,p_{\mathrm{KV}}}:\mathcal X_C\to\mathcal Y_C\) as model execution
under policy \(a\) and cache format \(p_{\mathrm{KV}}\). For the all-full policy,
write \(F^{\mathrm F}_{C,p_{\mathrm{KV}}}
:=F_{a^{\mathrm F},C,p_{\mathrm{KV}}}\). Let \(P_C\) be a workload
distribution on \(\mathcal X_C\), and let
\(d_C:\mathcal Y_C\times\mathcal Y_C\to\mathbb R_{\ge0}\) be a nonnegative
end-to-end distortion measure. We abbreviate the induced per-input distortion
as \(\delta_{C,p_{\mathrm{KV}}}(a,x):=
d_C(F^{\mathrm F}_{C,p_{\mathrm{KV}}}(x),
F_{a,C,p_{\mathrm{KV}}}(x))\). Among all static matrices whose normalized
residency does not exceed a cap \(\beta\), the ideal policy minimizes expected
deviation from the full-cache model:
\begin{equation}
 D^\star_{C,P_C,p_{\mathrm{KV}}}(\beta)
 =\min_{\substack{a\in\mathcal W^{|\mathcal U|}\\
                  R_C(a;p_{\mathrm{KV}})\le\beta}}
   \mathbb E_{x\sim P_C}
   \bigl[\delta_{C,p_{\mathrm{KV}}}(a,x)\bigr].
 \label{eq:budget-rd}
\end{equation}
We set \(D^\star_{C,P_C,p_{\mathrm{KV}}}(\beta)=+\infty\) when no policy meets
the cap. Directly optimizing Eq.~\eqref{eq:budget-rd} is impractical: the search
space contains \((K+1)^{|\mathcal U|}\) policy matrices, each requiring
end-to-end evaluation. We therefore propose SeqCalib, an efficient offline
allocation algorithm that replaces this global search with conditional
per-head similarity tests. The separate supplementary document motivates
heterogeneous windows and execution-order calibration.

\subsection{Sequential Prefix-Conditioned Calibration}
\label{sec:sequential-calibration}

SeqCalib processes layers in execution order and selects one history length per
KV head while conditioning later layers on earlier decisions.
This procedure corresponds to the center panel of
Figure~\ref{fig:headwisekv-overview}: SeqCalib freezes the deployed prefix,
replays the candidate windows, and selects the shortest window that satisfies
the similarity floor.

For layer \(\ell\), \(a_{<\ell}\) denotes the fixed policy on its configurable
predecessors. SeqCalib executes this prefix, traces layer \(\ell\) once with
full history, and replays every candidate suffix from the same trace. Candidates
are thus evaluated on deployed-prefix representations without a joint search
over all heads.

The comparison is made on attention outputs. For sequence \(x\), layer
\(\ell\), physical KV head \(h\), sampled query position \(t\), query head
\(q\in\mathcal G_{\ell h}\), and candidate window \(w\), let \(s\) index a
key/value position. The available values of \(s\) are
\(\mathcal J_t(w)=\{\max(1,t-w+1),\ldots,t\}\) for finite \(w\), and
\(\mathcal J_t(\mathrm{full})=\{1,\ldots,t\}\). The traced post-transform
query, key, and value vectors are denoted by
\(\mathbf q_{\ell q,t}\), \(\mathbf k_{\ell h,s}\), and
\(\mathbf v_{\ell h,s}\), respectively. Under \(a_{<\ell}\), their attention
logit is
\(z_{\ell hqts}=\operatorname{score}_{\ell}
(\mathbf q_{\ell q,t},\mathbf k_{\ell h,s};t,s)\) using the deployed
attention-logit function. For exact scaled dot product,
\(z_{\ell hqts}=\gamma_\ell\mathbf q_{\ell q,t}^{\top}
\mathbf k_{\ell h,s}\). We define
\(\pi_{\ell hqts}^{(w)}\) as the normalized attention weight assigned to
key/value position \(s\) when the query at \(t\) attends only to candidate
window \(w\). Let
\(Z_{\ell hqt}^{(w)}=
\sum_{j\in\mathcal J_t(w)}\exp(z_{\ell hqtj})\)
denote its normalizer. The corresponding attention output is
\(\mathbf o_{\ell hq}^{(w)}(x,t\mid a_{<\ell})\):
\begin{equation}
 \pi_{\ell hqts}^{(w)}
 =\frac{\exp(z_{\ell hqts})}{Z_{\ell hqt}^{(w)}}.
 \label{eq:conditional-attention-weight}
\end{equation}
\begin{equation}
 \mathbf o_{\ell hq}^{(w)}(x,t\mid a_{<\ell})
 =\sum_{s\in\mathcal J_t(w)}
   \pi_{\ell hqts}^{(w)}\mathbf v_{\ell h,s}.
 \label{eq:conditional-attention}
\end{equation}
SeqCalib scores a candidate by how closely this output matches the conditional
full-history output. For indexed calibration data
\(\mathcal D_{\mathrm{cal},C}=(x_n)_{n=1}^{N_{\mathrm{cal}}}\), let the
sampled query positions for \(x_n\) be \(\mathcal Q_C(x_n)\), and define the
nonempty index set
\(\mathcal I_{\ell h}=\{(n,t,q):1\le n\le N_{\mathrm{cal}},
t\in\mathcal Q_C(x_n),
q\in\mathcal G_{\ell h}\}\). Abbreviate
\(\mathbf o_{n\ell hqt}^{(w)}
=\mathbf o_{\ell hq}^{(w)}(x_n,t\mid a_{<\ell})\), and let
\(c_{n\ell hqt}^{(w)}=
\operatorname{cos}_{\mathrm{eval}}
(\mathbf o_{n\ell hqt}^{(w)},
 \mathbf o_{n\ell hqt}^{(\mathrm{full})})\):
\begin{equation}
 \widehat S_{\ell h}(w\mid a_{<\ell};C,p_{\mathrm{KV}})
 =\frac{1}{|\mathcal I_{\ell h}|}
 \sum_{(n,t,q)\in\mathcal I_{\ell h}}c_{n\ell hqt}^{(w)},
 \label{eq:conditional-similarity}
\end{equation}
where
\(c_{n\ell hqt}^{(w)}\) is the individual score returned by
\(\operatorname{cos}_{\mathrm{eval}}\),
\(\widehat S_{\ell h}\) is the mean score for KV head \(h\), and
\(\widehat\Phi_g\) is the lowest selected-head mean in run \(g\).  The full
candidate has score one.
The separate supplementary document specifies stabilization, sampling, and the
trace-compatibility check.

Given a similarity floor \(\tau\in(0,1]\), SeqCalib selects the lowest-cost
window that clears the floor for each head in the current layer:
\begin{equation}
\begin{aligned}
 \widehat a_{\ell h,\tau}&\in
\arg\min_{w\in\mathcal W}
 \bigl\{n_{\ell h,C}(w;p_{\mathrm{KV}}):\\[-1mm]
 &\hspace{15mm}
 \widehat S_{\ell h}
 (w\mid\widehat a_{<\ell,\tau};C,p_{\mathrm{KV}})
 \ge\tau\bigr\}.
\end{aligned}
 \label{eq:threshold-selection}
\end{equation}
Because \(\mathrm{full}\in\mathcal W\), every stage is feasible, and
\(\prec_{\mathcal W}\) breaks equal-cost ties. All heads in a layer use the
same prefix. Freezing their selections before advancing gives minimum feasible
layer cost conditional on that realized prefix, not global optimality over
\(\mathcal W^{|\mathcal U|}\). The separate supplementary document gives the formal
proposition and proof.

One layerwise pass yields a static policy and its normalized residency:
\begin{equation}
\begin{aligned}
 \widehat a_\tau&=\operatorname{SeqCalib}
 (\mathcal D_{\mathrm{cal},C},\mathcal W,\tau,C,p_{\mathrm{KV}}),\\
 \widehat\beta_\tau&=R_C(\widehat a_\tau;p_{\mathrm{KV}}).
\end{aligned}
 \label{eq:seqcalib}
\end{equation}
To support different memory caps, we run a finite ordered grid
\(\mathcal T=(\tau_1,\ldots,\tau_G)\), with
\(\{\tau_g\}_{g=1}^G\subset(0,1]\), independently and record each policy, rate,
floor, threshold, and grid index. Among records satisfying budget \(\beta\),
the exact lexicographic rule maximizes the recorded floor, then minimizes rate,
and finally minimizes grid index. Only the listed order of \(\mathcal T\)
determines this final deterministic tie-break. Floors from different runs use
different lower prefixes and are therefore not common-reference fidelity.
rates need not be monotone in \(\tau\), so the grid is enumerated.
The separate supplementary document defines the annotated records and selector
exactly.

\subsection{From Policy to Physical Residency}
\label{sec:runtime}

We realize the selected history matrix as physical cache residency rather than
as a logical mask over an all-full allocation. The serving runtime loads the
matrix once and does not recompute it from the prompt, request history, or online
attention scores. Each configurable KV head writes to and reads from the cache
group associated with its assigned window. Heads assigned the same window may
share a group configuration, but their KV contents remain disjoint. The query
heads in a GQA group read their shared physical KV-head slice, after which the
per-head outputs are concatenated and passed through the model's unchanged
gating and output projection. Native local and recurrent layers remain on their
original cache paths.
The right panel of Figure~\ref{fig:headwisekv-overview} illustrates this
transition from the offline policy matrix to grouped physical KV allocations
in the serving runtime.

Calibration and deployment use the same layer--head indexing, window semantics,
and attention scaling. Before accepting a trace, the calibrator reconstructs a
small full-window subset to detect incompatible layouts or GQA mappings. This
check validates the execution contract, not the quality of a finite-window
policy. Because only the configured histories are allocated, the runtime reduces
the KV component counted by Eq.~\eqref{eq:nominal-storage}. Total GPU memory also
contains weights, workspaces, and allocator overhead and is therefore measured
separately. The separate supplementary document specifies serialization, graph
routing, backend configuration, and the graphics processing unit (GPU) memory
accounting boundary.

%% file: sections/evaluation.tex
\section{Evaluation}
\label{sec:evaluation}

{
We evaluate HeadWiseKV in terms of generation quality,
serving efficiency, and transfer across model families and scales. We compare
baselines and compression levels in a controlled fixed-model study, then
assess transfer on additional models.

\input{sections/qwen27_baseline_table}

\subsection{Experimental Setup and Baselines}
\label{sec:initial-config}

\paragraph{Models and cache policies.}
We evaluate four instruction-tuned hybrid models:
Qwen3.5-9B, Qwen3.6-27B~\citep{qwen36_27b}, Qwen3.6-35B-A3B, and
Gemma4-31B. Unless stated otherwise, weights use Q4\_K\_M quantization and KV
caches use 16-bit floating point. Full-KV preserves every configurable global
history and leaves native bounded-memory paths unchanged. HeadWiseKV uses
\(C_{\mathrm{cal}}=131072\) and
\(\mathcal W=\{8\mathrm K,16\mathrm K,32\mathrm K,\mathrm{full}\}\). We report
the configurable-global retention \(R_{131072}\) from
Eq.~\eqref{eq:normalized-rate}. The shared \(\tau=.995\) setting can therefore
induce different retention rates across models. SeqCalib uses a contiguous
WikiText-103 stream~\citep{wikitext} and 256 unique query positions sampled
from the final quarter of the calibration context. Calibration activates each
accepted lower-layer decision before processing the next global layer. All
calibration runs use F16 KV and full model offload. The supplementary material
provides the remaining configuration and reproducibility details.

\paragraph{Tasks and metrics.}
RULER~\citep{ruler} measures long-context retrieval over
five needle-in-a-haystack tasks at six context lengths from 4K to 128K. Each
configuration contains 1,200 examples, and we report mean item score on a
0--100 scale. LoCoMo~\citep{locomo} evaluates conversational memory on 1,540
questions from categories 1--4 with an evidence-aware correctness judge. We follow the official LoCoMo evaluation protocol, changing only the judge model to DeepSeek V4 Flash.
The Airline and Retail domains originate from
\(\tau\)-bench~\citep{yao2024taubench} and are evaluated with the
\(\tau^2\)-Bench framework~\citep{barres2025tau2}. LongMemEval-S~\citep{longmemeval} is reported in the
supplementary material. Scores are compared only within the same benchmark
and model.

\paragraph{Systems measurements.}
The fixed-model study evaluates
\(C\in\{8,16,32,64,96,112,128\}\)K on an RTX 4090 D. Each fresh-server run
receives an exact \(C-128\)-token prompt and greedily generates 128 tokens.
Every successful configuration has three repetitions. We report throughput
pooled over timed tokens and the median sampled whole-device peak memory. A
separate adjacent-grid probe identifies the last successful context and first
out-of-memory (OOM) context for Full-KV and HeadWiseKV. Contexts above 128K
are capacity probes rather than quality evaluations. The supplementary
material provides runtime flags and OOM validation procedures.

\paragraph{Baselines and comparison protocol.}
Full-KV is the uncompressed control. AdaKV~\citep{adakv}
and HeadKV-R2~\citep{headkv} provide prompt-adaptive head-level baselines,
StreamingLLM~\citep{streamingllm} provides a fixed sink--recent baseline, and
DuoAttention~\citep{duoattention} represents binary retrieval/streaming head
assignment. Only Full-KV and HeadWiseKV belong to the counterbalanced matched
systems cohort, so direct memory and throughput ratios are restricted to this
pair. Other systems campaigns are reported as absolute measurements. The
DuoAttention systems implementation uses a binary projection of the
HeadWiseKV profile. The StreamingLLM control uses a matched
\(4+90{,}364\) cache and is independently measured only after eviction
becomes active. Implementation qualifications are provided in the
supplementary material.

\paragraph{Comparison scope.}
Quality comparisons use the same model, benchmark, and
retention definition. Systems ratios are reported only for measurements from
the matched Full-KV and HeadWiseKV cohort. This separation prevents
differences in runtime implementation or measurement campaign from being
interpreted as effects of the cache policy.

\subsection{Comparison with KV-Cache Baselines}
\label{sec:baseline-comparison}

\paragraph{Downstream quality.}
We compare downstream quality at a matched 68.95\% KV-retention budget in
Table~\ref{tab:baseline-quality}. HeadWiseKV is the only compressed method
that consistently preserves Full-KV behavior across long-context retrieval
(RULER), conversational memory (LoCoMo), and agent tasks
(the Airline and Retail domains of \(\tau^2\)-Bench). AdaKV and HeadKV-R2
preserve conversational-memory quality, but their retrieval performance is
uneven, with a particularly large degradation for HeadKV-R2. StreamingLLM's
sink--recent cache retains reasonable retrieval quality but loses more on
conversational memory and both agent domains. Among the baselines,
DuoAttention achieves the highest Retail task success (82.16), but still
trails HeadWiseKV on retrieval, conversational memory, and both agent domains.
HeadWiseKV remains near Full-KV on retrieval and conversational
memory while achieving higher task success than Full-KV in both agent
domains. It suggests that allocation granularity matters:
HeadWiseKV may reduce distracting history while retaining the long-range
context needed by sensitive heads, helping the model focus on task-relevant
evidence.
%Matching the overall retention budget alone does not provide
%consistent quality across workloads.

\input{sections/qwen27_operating_points_table}

\begin{figure*}[t]
\centering
\includegraphics[width=\textwidth]{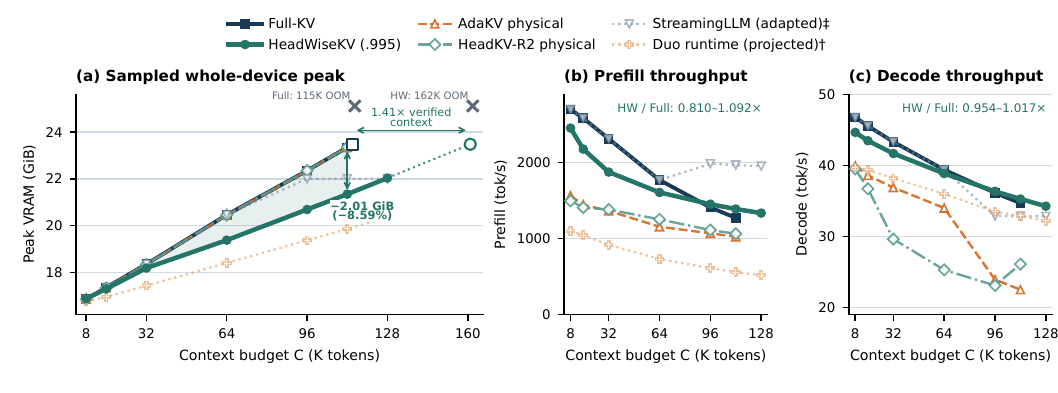}
\caption{{\textbf{Systems scaling for the fixed-model case study}
(Qwen3.6-27B Q4\_K\_M, F16 KV). Panels report \textbf{sampled peak VRAM},
\textbf{prefill throughput}, and \textbf{decode throughput} across context
lengths. Filled Full-KV and
HeadWiseKV curves form the matched cohort; the remaining baselines are
absolute campaign traces.  In Panel (a), dotted extensions show independent
capacity successes and gray crosses mark adjacent OOM tests.
For StreamingLLM (\(4+90{,}364\)), points at
\(C\leq90{,}368\) reuse the matched Full-KV measurements because no eviction
is triggered and are not independent measurements; the 96K, 112K, and 128K
points are measured with our adapted Qwen/llama.cpp fixed-cache control.}}
\label{fig:systems-efficiency-curves}
\end{figure*}

\paragraph{Memory and throughput.}
The matched Full-KV and HeadWiseKV sweep in
Figure~\ref{fig:systems-efficiency-curves} confirms that the nominal cache
reduction becomes a physical memory benefit. At 112K, HeadWiseKV lowers
sampled peak memory by 8.59\%, while prefill and decode throughput are
\(1.092\times\) and \(1.017\times\) Full-KV. Decode throughput remains close
to Full-KV across their common successful contexts, although prefill is more
context-dependent.
%%The main systems benefit is therefore
%not an isolated speedup, but a lower memory trajectory that preserves decoding
%efficiency as context grows.
In contrast, the prompt-adaptive AdaKV and HeadKV-R2 baselines make
their retention decisions after constructing the prefill cache. Their
post-prefill compaction can reduce decoding-time residency, but it does not
lower the sampled prefill peak in Figure~\ref{fig:systems-efficiency-curves}.
%This distinction explains why matching HeadWiseKV's retained fraction does
%not give these methods the same admission-memory benefit.

\paragraph{Static systems controls.}
We compare the serving behavior of StreamingLLM and the
projected DuoAttention runtime in
Figure~\ref{fig:systems-efficiency-curves}, with their 128K results summarized
in Table~\ref{tab:baseline-quality}. Once eviction is active, StreamingLLM
maintains high prefill throughput but has lower 128K decode throughput than
the reported HeadWiseKV trace. DuoAttention uses less memory at 128K, but its
prefill throughput is substantially lower. Because the controls come from
separate campaigns, they do not establish complete joint quality--systems
operating points.
HeadWiseKV provides the most complete measured balance of downstream quality,
physical memory reduction, decode throughput, and context capacity. Because
the controls come from separate campaigns, we report their measurements as
absolute values and do not claim cross-campaign speedups.

\paragraph{Maximum context.}
We compare the maximum admissible context of Full-KV and
HeadWiseKV using the adjacent-grid capacity probes in
Figure~\ref{fig:systems-efficiency-curves}. Full-KV succeeds at 114K and OOMs
at 115K, whereas HeadWiseKV succeeds at 161K and OOMs at 162K. The maximum
verified successful context is therefore \(1.41\times\) larger on the tested
1K grid.

\subsection{Evaluation Across Models and Scales}
\label{sec:compression-tradeoffs}

\paragraph{Retention sensitivity.}
We evaluate HeadWiseKV across four KV-retention levels in
Table~\ref{tab:compression-tradeoffs}. HeadWiseKV preserves Full-KV quality
over a broad range of compression levels. RULER and LoCoMo remain unchanged
at 96.00\% and 85.74\% retention. At the lowest-retention point, 68.95\%,
RULER decreases by 0.35 points and LoCoMo remains unchanged. Quality
degradation is therefore small and appears only at the strongest evaluated
compression on these benchmarks. The flat
moderate-compression region indicates that the
global KV heads contain substantial heterogeneous redundancy that can be
removed before downstream quality changes measurably.
% It also supports
%multilevel allocation over a single uniform window: heads that need long
%history can retain it while shorter windows absorb most of the compression.

%\subsection{Generalization Across Model Families and Scales}
%\label{sec:cross-model}

\input{sections/quality_table}

\paragraph{Model transfer.}
We evaluate cross-model transfer on four hybrid models in
Table~\ref{tab:quality-main}. The \(\tau=.995\) policies remain within 0.61
RULER points of Full-KV despite producing different retention rates. Transfer
is less uniform on LoCoMo. Two models remain unchanged or improve slightly,
while the other two lose at most 2.08 points. SeqCalib
therefore transfers more reliably on retrieval than on conversational
long-memory questions. The wide variation in induced
retention also confirms that \(\tau\) acts as a fidelity threshold rather than
a model-independent compression ratio. This distinction matters in
deployment. A shared threshold can express a common calibration criterion,
but capacity planning must use the policy produced for each model rather than
assume a fixed retention rate. The model-dependent LoCoMo behavior further
indicates that the final operating point should be validated on the target
workload, even when calibration similarity and retrieval results are stable.
}

%% file: sections/qwen27_baseline_table.tex
{
\begin{table*}[t]
\centering
\caption{{\textbf{Fixed-model downstream and 128K systems comparison on Qwen3.6-27B Q4\_K\_M.}
\(\tau^2\)-Bench uses no-thinking decoding. \textbf{Higher values are
better for all metrics except \(R^{Q}_{128\mathrm K}\).}
``Static'' denotes a capacity and selection rule fixed before serving, whereas
``Dynamic'' denotes request-conditioned score-based selection.
AdaKV and HeadKV-R2 use context-wise KV budgets matched to the HeadWiseKV
\(\tau=.995\) quality policy.  The final two columns report pooled throughput
at 128K over three repetitions on an RTX 4090 D; OOM denotes
failed admission in all attempts.  These systems entries are absolute
rather than matched quality--throughput comparisons. }}
\label{tab:baseline-quality}
\small
\setlength{\tabcolsep}{0pt}
\renewcommand{\arraystretch}{1.06}
\begin{tabular*}{\textwidth}{@{\extracolsep{\fill}}lcrrrrrrr@{}}
\toprule
\multirow{2}{*}{Method} &
\multirow{2}{*}{Policy} &
\multirow{2}{*}{\(\boldsymbol{R^{Q}_{128\mathrm K}}\)} &
\multirow{2}{*}{RULER} &
\multirow{2}{*}{LoCoMo} &
\multicolumn{2}{c}{\(\tau^2\)-Bench} &
\multicolumn{2}{c}{128K Throughput} \\
\cmidrule(lr){6-7}\cmidrule(l){8-9}
& & & & & Airline & Retail & Prefill (tok/s) & Decode (tok/s) \\
\midrule
Full-KV & Static & 100.00\% & 100.00 & 94.55 & 77.00 & 82.02 &
\multicolumn{2}{c}{OOM} \\
\midrule
AdaKV~\citep{adakv} & Dynamic & 68.95\% & 92.83 & 94.22 & 73.50 &
78.73 &
\multicolumn{2}{c}{OOM} \\
HeadKV-R2~\citep{headkv} & Dynamic & 68.95\% & 39.65 & 94.22 & 70.50 &
79.39 &
\multicolumn{2}{c}{OOM} \\
StreamingLLM~\citep{streamingllm} &
Static & 68.95\% &
83.88 & 82.40 &
73.27 &
80.26 &
\textbf{1952.98} & 32.81 \\
DuoAttention~\citep{duoattention} &
Static & 68.95\%  & 88.58 & 82.60 &
73.00 &
82.16 &
517.03 & 32.16 \\
\midrule
\textbf{HeadWiseKV (\(\tau=.995\))} &
Static & 68.95\% &
\textbf{99.65} &
\textbf{94.55} &
\textbf{79.50} &
\textbf{82.24} &
1334.96 & \textbf{34.24} \\
\bottomrule
\end{tabular*}
\end{table*}
}

%% file: sections/qwen27_operating_points_table.tex
{
\begin{table}[!t]
\centering
\caption{\textbf{Quality across fixed-model HeadWiseKV operating points.}
Retention is measured over configurable global KV heads at 128K.}
\label{tab:compression-tradeoffs}
\small
\setlength{\tabcolsep}{0pt}
\renewcommand{\arraystretch}{1.06}
\begin{tabular*}{\columnwidth}{@{\extracolsep{\fill}}lrrr@{}}
\toprule
Setting &
\(R^{Q}_{128\mathrm K}\) &
RULER &
LoCoMo \\
\midrule
Full-KV & 100.00\% & 100.00 & 94.55 \\
HeadWiseKV, \(\tau=.999\) & 96.00\% & 100.00 & 94.55 \\
HeadWiseKV, \(\tau=.998\) & 85.74\% & 100.00 & 94.55 \\
\textbf{HeadWiseKV, \(\boldsymbol{\tau=.995}\)} & 68.95\% & 99.65 & 94.55 \\
\bottomrule
\end{tabular*}
\end{table}
}

%% file: sections/quality_table.tex
{
\begin{table}[t]
\centering
\caption{{\textbf{Cross-model downstream quality of HeadWiseKV.} Results use
the shared \(\tau=.995\) setting with Q4\_K\_M weights. Arrows report
Full-KV \(\to\) HeadWiseKV. Because \(\tau\) induces model-specific
retention, all comparisons are within model.}}
\label{tab:quality-main}
\small
\renewcommand{\arraystretch}{1.04}

\setlength{\tabcolsep}{2pt}
\begin{tabular}{@{}lrrr@{}}
\toprule
\multirow{2}{*}{Model} &
\multirow{2}{*}{\(R^{Q}_{128\mathrm K}\)} &
RULER &
LoCoMo \\
\cmidrule(lr){3-3}\cmidrule(l){4-4}
& & Full \(\to\) HW & Full \(\to\) HW \\
\midrule
{Qwen3.6-27B} & 68.95\% & 100.00 \(\to\) 99.65 & 94.55 \(\to\) 94.55 \\
Qwen3.6-35B-A3B   & 74.38\% & 100.00 \(\to\) 99.94 & 92.99 \(\to\) 90.91 \\
Gemma4-31B        & 91.89\% & 99.96 \(\to\) 99.50 & 94.87 \(\to\) 94.94 \\
Qwen3.5-9B        & 58.01\% & 99.90 \(\to\) 99.29 & 91.30 \(\to\) 89.74 \\
\bottomrule
\end{tabular}
\end{table}
}

%% file: sections/conclusion.tex
\section{Conclusion}

We introduced HeadWiseKV, a training-free framework for
predictable physical KV residency in hybrid long-context models. The core component SeqCalib
assigns multilevel per-head suffix windows under the lower-layer decisions
used at deployment, and the grouped-cache runtime materializes this policy
without retraining. Across the evaluated models, HeadWiseKV preserves Full-KV
quality more consistently than matched-retention baselines. In the fixed-model
systems study, it reduces physical memory, extends the verified context range,
and maintains decode throughput.

%%As a prompt-independent suffix policy, HeadWiseKV cannot recover evidence
%outside its retained windows. SeqCalib does not guarantee downstream quality
%or global optimality, and new deployment regimes may require recalibration.
%Future work will combine predictable physical allocation with lightweight
%request-specific adaptation and evaluate complete multi-\(\tau\) selection
%across model families.

%% file: sections/theory_appendix.tex
\section{Design Rationale and Calibration Properties}
\label{app:theory}

\subsection{Per-Head Retention and Sequential Calibration}

Equation~\eqref{eq:budget-rd} defines the ideal allocation over per-head window
matrices. SeqCalib replaces the joint search with tractable conditional
decisions: it chooses one window for every physical KV head while processing
layers in execution order under the realized lower-layer policy. This design
reflects two properties of cache truncation: its local effect can vary across
heads, and earlier choices can change the states seen by later blocks.

\paragraph{Head-dependent truncation.}
Consider a fixed traced attention computation for one query head. Let
\(J_{\mathrm F}\) be the finite set of unmasked
full-history positions, with attention weights \(\alpha_j\ge0\),
\(\sum_{j\in J_{\mathrm F}}\alpha_j=1\), and values \(v_j\). Let the retained
suffix \(J_w\subseteq J_{\mathrm F}\) be nonempty, and write
\(m_w=\sum_{j\in J_{\mathrm F}\setminus J_w}\alpha_j\). For
\(0<m_w<1\), define the retained and omitted means
\[
 \mu_w=\frac{1}{1-m_w}\sum_{r\in J_w}\alpha_rv_r,
 \qquad
 \mu_{\bar w}=\frac{1}{m_w}
 \sum_{s\in J_{\mathrm F}\setminus J_w}\alpha_sv_s,
\]
and let
\(\Delta_V=\max_{r\in J_w,s\in J_{\mathrm F}\setminus J_w}
\lVert v_r-v_s\rVert_2\).
The full output is
\(o_{\mathrm F}=\sum_{j\in J_{\mathrm F}}\alpha_jv_j
=(1-m_w)\mu_w+m_w\mu_{\bar w}\).
Renormalizing the same logits on \(J_w\) gives \(o_w=\mu_w\), and therefore
\begin{equation}
 \lVert o_{\mathrm F}-o_w\rVert_2
 =m_w\lVert\mu_{\bar w}-\mu_w\rVert_2
 \le m_w\Delta_V.
 \label{eq:local-attention-bound}
\end{equation}
Writing
\(\mu_w-\mu_{\bar w}=
\sum_{r\in J_w}\sum_{s\in J_{\mathrm F}\setminus J_w}
\gamma_r\xi_s(v_r-v_s)\), where
\(\gamma_r=\alpha_r/(1-m_w)\) and \(\xi_s=\alpha_s/m_w\), shows that the
difference is a convex combination of cross-set value differences and hence
has norm at most \(\Delta_V\). The case \(m_w=0\) is exact, while a nonempty
suffix containing a finite unmasked logit ensures \(m_w<1\).

Both factors can vary across the calibration trace: \(m_w\) with the query and
attention pattern, and \(\Delta_V\) with the retained and omitted values of the
physical KV head. A fixed suffix can therefore produce different local
deviations across heads. SeqCalib accounts for this heterogeneity by averaging
over the query heads in each GQA group and assigning their shared physical KV
head its own window.

\paragraph{Why calibration follows execution order.}
Per-head allocation is coupled across depth: shortening an earlier cache can
change the state presented to every later block. We write this dependence as a
blockwise perturbation recurrence. For block \(b\), let \(x_b^{\mathrm F}\)
and \(x_b^a\) be the full and policy states. We view them in a common masked
state space by padding omitted history coordinates; the mask distinguishes
these coordinates from resident cache entries, leaving runtime allocation
unchanged. Let \(T_b\) and \(T_b^a\) be the corresponding block maps, so
\(x_{b+1}^{\mathrm F}=T_b(x_b^{\mathrm F})\) and
\(x_{b+1}^a=T_b^a(x_b^a)\).

Let \(\mathcal X_b\) contain the paired states encountered at block \(b\), with
both maps defined on this set. Suppose \(T_b\) is \(L_b\)-Lipschitz on
\(\mathcal X_b\), for finite \(L_b\ge0\), and assume
\[
 \epsilon_b^{\mathrm{op}}(a)
 :=\sup_{x\in\mathcal X_b}\lVert T_b(x)-T_b^a(x)\rVert_2<\infty.
\]
For \(\Delta_b=\lVert x_b^{\mathrm F}-x_b^a\rVert_2\), the triangle
inequality gives
\[
\begin{aligned}
 \Delta_{b+1}
 &=\lVert T_b(x_b^{\mathrm F})-T_b^a(x_b^a)\rVert_2\\
 &\le\lVert T_b(x_b^{\mathrm F})-T_b(x_b^a)\rVert_2
     +\lVert T_b(x_b^a)-T_b^a(x_b^a)\rVert_2\\
 &\le L_b\Delta_b+\epsilon_b^{\mathrm{op}}(a).
\end{aligned}
\]
Thus the state used by a later block depends recursively on the retention
decisions already active below it. SeqCalib follows the same dependency: it
calibrates layer \(\ell\) with the selected lower-layer windows in place, so
candidates see the prefix used at deployment.

\subsection{Calibration Score and Fixed-Prefix Selection}
\label{app:calibration}

SeqCalib compares candidate and full-history outputs with a stabilized cosine
score. For \(\epsilon_{\mathrm{cos}}=10^{-12}\), define
\[
 \widetilde u=
 \begin{cases}
 0,&\lVert u\rVert_2\le\epsilon_{\mathrm{cos}},\\
 u,&\text{otherwise},
 \end{cases}
\]
and define \(\widetilde v\) in the same way. Let \(\mathsf{xor}\) denote that
exactly one stabilized vector is zero. Then
\begin{equation}
\operatorname{cos}_{\mathrm{eval}}(u,v)=
 \begin{cases}
1,&\widetilde u=\widetilde v=0,\\
0,&(\widetilde u=0)\ \mathsf{xor}\ (\widetilde v=0),\\
\dfrac{\widetilde u^\top\widetilde v}
{\lVert\widetilde u\rVert_2\lVert\widetilde v\rVert_2},
&\text{otherwise}.
\end{cases}
\label{eq:stabilized-cosine}
\end{equation}
Equation~\eqref{eq:conditional-similarity} averages this score over the
nonempty calibration index set \(\mathcal I_{\ell h}\) defined in
Section~\ref{sec:sequential-calibration}. By definition, the full candidate
scores one, and its replay supplies the reference for both trace compatibility
and finite-window comparisons.

Fix a calibrated unit \((\ell,h)\), floor \(\tau\), and realized prefix
\(\widehat a_{<\ell,\tau}\). Consider any finite candidate
\(w\ne\mathrm{full}\) feasible in Eq.~\eqref{eq:threshold-selection}. Let
\(n=|\mathcal I_{\ell h}|\), and for \(r=(i,t,q)\in\mathcal I_{\ell h}\) set
\[
\begin{aligned}
 y_r^w
 &=\mathbf o_{\ell hq}^{(w)}
   (x_i,t\mid\widehat a_{<\ell,\tau}),\\
 y_r^{\mathrm F}
 &=\mathbf o_{\ell hq}^{(\mathrm{full})}
   (x_i,t\mid\widehat a_{<\ell,\tau}).
\end{aligned}
\]
Let \(\widetilde y_r^w,\widetilde y_r^{\mathrm F}\) be their stabilized
versions from Eq.~\eqref{eq:stabilized-cosine}, and abbreviate
\[
 \widehat S_w=\widehat S_{\ell h}
 (w\mid\widehat a_{<\ell,\tau};C,p_{\mathrm{KV}}).
\]
The definition of \(\widehat S_{\ell h}\) and feasibility of \(w\) give
\[
 \frac{1}{n}\sum_{r\in\mathcal I_{\ell h}}
 \operatorname{cos}_{\mathrm{eval}}(y_r^w,y_r^{\mathrm F})
 =\widehat S_w\ge\tau.
\]
\paragraph{Calibration-trace error bound.}
Define the largest stabilized norm \(B_w\) and the largest paired norm
difference \(\eta_w\) by
\[
\begin{aligned}
 B_w
 &=\max_{r\in\mathcal I_{\ell h}}
   \max\{\lVert\widetilde y_r^w\rVert_2,
           \lVert\widetilde y_r^{\mathrm F}\rVert_2\},\\
 \eta_w
 &=\max_{r\in\mathcal I_{\ell h}}
   \left|\lVert\widetilde y_r^w\rVert_2
          -\lVert\widetilde y_r^{\mathrm F}\rVert_2\right|.
\end{aligned}
\]
These quantities give
\begin{equation}
\begin{aligned}
 &\frac{1}{n}\sum_{r\in\mathcal I_{\ell h}}
  \lVert\widetilde y_r^w-\widetilde y_r^{\mathrm F}\rVert_2^2\\
 &\quad\le \eta_w^2+2B_w^2(1-\widehat S_w)\\
 &\quad\le \eta_w^2+2B_w^2(1-\tau).
\end{aligned}
 \label{eq:conditional-empirical-cosine-bound}
\end{equation}

\noindent\emph{Derivation.}
For nonzero stabilized vectors \(u\) and \(v\),
\[
 \lVert u-v\rVert_2^2
 =\bigl(\lVert u\rVert_2-\lVert v\rVert_2\bigr)^2
 +2\lVert u\rVert_2\lVert v\rVert_2
  \bigl(1-\operatorname{cos}_{\mathrm{eval}}(u,v)\bigr).
\]
The stabilized zero cases satisfy the same upper bound by
Eq.~\eqref{eq:stabilized-cosine}. Applying the identity pointwise and averaging
gives the first inequality in
Eq.~\eqref{eq:conditional-empirical-cosine-bound}; feasibility gives the second.

Thus every feasible finite window satisfies an in-trace squared-error bound on
stabilized outputs; the full window has zero error.
For fixed \(B_w\), increasing \(\tau\) tightens only the
angular term; it does not control \(\eta_w\).

Fix \(C\), \(p_{\mathrm{KV}}\), and \(\tau\in(0,1]\). Assume
\(1\le|\mathcal U|<\infty\), that \(\mathcal W\) is finite and contains
\(\mathrm{full}\), and that \(0<|\mathcal I_{\ell h}|<\infty\) for every
decision unit.
\begin{proposition}[Fixed-prefix minimality]
\label{prop:seqcalib}
SeqCalib is feasible and terminates after finitely many candidate evaluations.
Conditional on the realized lower-layer prefix, the vector of windows selected
at each layer has minimum nominal layer cost among same-layer choices satisfying
all local similarity floors.
\end{proposition}
\noindent\emph{Proof.}
For head \(h\) in layer \(\ell\), let \(\mathcal F_{\ell h}\subseteq\mathcal W\)
be the windows satisfying its conditional similarity floor under the fixed
lower prefix. The full candidate has score one, so every
\(\mathcal F_{\ell h}\) is nonempty. The codebook and decision-unit set are
finite, which proves termination. With the lower prefix and current-layer
tensors fixed, each constraint depends only on its own head window and the
nominal layer cost is additive:
\[
 N_{\ell,C}(a_\ell;p_{\mathrm{KV}})
 =\sum_{h\in\mathcal H_\ell}
 n_{\ell h,C}(a_{\ell h};p_{\mathrm{KV}}).
\]
The feasible same-layer set is therefore
\(\prod_{h\in\mathcal H_\ell}\mathcal F_{\ell h}\). Equation~\eqref{eq:threshold-selection}
minimizes each summand over its corresponding factor, so summing the per-head
inequalities proves minimum feasible layer cost. The order
\(\prec_{\mathcal W}\) selects a deterministic representative among equal-cost
minimizers.\hfill\(\square\)

\begin{corollary}[Cost relative to a shared window]
\label{cor:shared-window}
For a SeqCalib run at \((C,p_{\mathrm{KV}},\tau)\), fix layer \(\ell\) and its
realized lower-layer prefix \(\widehat a_{<\ell,\tau}\). If a shared window
\(w\in\mathcal W\) satisfies
\[
\widehat S_{\ell h}
(w\mid\widehat a_{<\ell,\tau};C,p_{\mathrm{KV}})\ge\tau
\qquad\text{for every }h\in\mathcal H_\ell,
\]
then
\[
\sum_{h\in\mathcal H_\ell}
n_{\ell h,C}(\widehat a_{\ell h,\tau};p_{\mathrm{KV}})
\le
\sum_{h\in\mathcal H_\ell}
n_{\ell h,C}(w;p_{\mathrm{KV}}).
\]
The inequality is strict if at least one head admits a feasible candidate
whose nominal cost is strictly lower than that of \(w\).
\end{corollary}

\noindent\emph{Proof.}
The shared candidate \(w\) belongs to every per-head feasible set.
Equation~\eqref{eq:threshold-selection} therefore selects a candidate whose
cost is no larger for each head. Summing these inequalities proves the claim;
a strictly lower feasible cost for any head makes the sum strict.\hfill\(\square\)

\paragraph{Finite-grid budget selection.}
Fix \(C\) and \(p_{\mathrm{KV}}\) with
\(N_C(a^{\mathrm F};p_{\mathrm{KV}})>0\), so the normalized rate is defined,
and take a nonempty ordered grid
\(\mathcal T=(\tau_1,\ldots,\tau_G)\), where \(G\ge1\) and
\(\tau_g\in(0,1]\). Let \(\widehat a^{(g)}\) be the independent SeqCalib run
at \(\tau_g\),
\(\widehat\beta_g=R_C(\widehat a^{(g)};p_{\mathrm{KV}})\), and set
\[
\widehat\Phi_g=
\min_{(\ell,h)\in\mathcal U}
\widehat S_{\ell h}\!\left(
\widehat a_{\ell h}^{(g)}
\mid\widehat a_{<\ell}^{(g)};C,p_{\mathrm{KV}}\right).
\]
Associate run \(g\) with the record
\((g,\tau_g,\widehat a^{(g)},\widehat\beta_g,\widehat\Phi_g)\), and collect
the records in \(\mathbf O_{\mathcal T}\). The matrix determines its rate,
while \(\tau_g\), \(\widehat\Phi_g\), and \(g\) identify the calibration
operating point used for offline selection; deployment uses the chosen matrix.
Given budget \(\beta\), define
\(\mathcal G_\beta^{\mathrm{feas}}
=\{g:\widehat\beta_g\le\beta\}\). If
\(\mathcal G_\beta^{\mathrm{feas}}\) is empty, the recorded grid contains no
operating point satisfying the budget. Otherwise select, in lexicographic order,
\begin{equation}
\begin{aligned}
 \Phi^\star&=\max_{g\in\mathcal G_\beta^{\mathrm{feas}}}\widehat\Phi_g,\\
 \widehat\beta^\star
 &=\min\{\widehat\beta_g:g\in\mathcal G_\beta^{\mathrm{feas}},\
                    \widehat\Phi_g=\Phi^\star\},\\
 g^\star
 &=\min\{g:g\in\mathcal G_\beta^{\mathrm{feas}},\
                    \widehat\Phi_g=\Phi^\star,\
                    \widehat\beta_g=\widehat\beta^\star\},\\
 \widehat a_\beta&=\widehat a^{(g^\star)}.
\end{aligned}
 \label{eq:budget-selector}
\end{equation}
Each \(\widehat\Phi_g\) is the lowest selected-head mean similarity evaluated
under run \(g\)'s realized lower-layer prefix. Since changing \(\tau_g\) can
alter early prefixes, later scores, and the resulting rate, the selector
enumerates the grid explicitly.

\subsection{Predictive Distortion under Log Loss}
\label{app:predictive-interpretation}

Equation~\eqref{eq:budget-rd} leaves the end-to-end distortion \(d_C\)
abstract. Under teacher-relative log loss, it has a concrete predictive
interpretation \citep{cover_thomas,courtade_weissman}. Fix \(C\),
\(p_{\mathrm{KV}}\), a common nonempty finite set of prediction positions
\(\mathcal T_C\), and a fixed position distribution
\(\pi_C\in\Delta(\mathcal T_C)\). Let \(\mathcal V\) be the finite vocabulary,
and let \(\Delta^\circ(\mathcal V)\) denote its strictly positive probability
simplex.

For input \(x\) and position \(t\), let \((X_t,S_t)\) be the prefix history
and common side information, including the position, determined by \((x,t)\).
Let \(Z_{a,t}=g_a(X_t,S_t)\) be the state obtained by rerunning policy \(a\).
The full and policy predictors are
\(p_{\mathrm F}(\cdot\mid X_t,S_t)\) and
\(q_a(\cdot\mid Z_{a,t},S_t)\), both in \(\Delta^\circ(\mathcal V)\). Take
\[
 \mathcal Y_C=\bigl(\Delta^\circ(\mathcal V)\bigr)^{\mathcal T_C}
\]
and define
\[
 d_C(\{r_t\},\{s_t\})
 =\sum_{t\in\mathcal T_C}\pi_{C,t}
   D_{\mathrm{KL}}(r_t\Vert s_t).
\]
This distortion is finite and nonnegative. With
\[
\begin{aligned}
 F^{\mathrm F}_{C,p_{\mathrm{KV}}}(x)
 &=\{p_{\mathrm F}(\cdot\mid X_t,S_t)\}_{t\in\mathcal T_C},\\
 F_{a,C,p_{\mathrm{KV}}}(x)
 &=\{q_a(\cdot\mid Z_{a,t},S_t)\}_{t\in\mathcal T_C},
\end{aligned}
\]
we obtain, with expectations interpreted in \([0,\infty]\),
\[
\begin{aligned}
 D_{\log}(a)
 &:=\mathbb E_{x\sim P_C}d_C\!\bigl(
 F^{\mathrm F}_{C,p_{\mathrm{KV}}}(x),
 \\[-1mm]
 &\hspace{32mm}F_{a,C,p_{\mathrm{KV}}}(x)\bigr)\\
 &=\mathbb E_{\substack{x\sim P_C\\J\sim\pi_C}}
 D_{\mathrm{KL}}\!\left(
 p_{\mathrm F}(\cdot\mid X_J,S_J)\Vert
 q_a(\cdot\mid Z_{a,J},S_J)\right).
\end{aligned}
\]

To interpret this distortion, draw
\(Y\sim p_{\mathrm F}(\cdot\mid X_J,S_J)\) and let
\(P_{\mathrm F}^a\) denote the resulting joint law. Abbreviate
\((X,S,Z_a)=(X_J,S_J,Z_{a,J})\), and define
\[
 \bar p_{\mathrm F}^a(y\mid z,s)=
 \mathbb E_{P_{\mathrm F}^a}
 [p_{\mathrm F}(y\mid X,S)\mid Z_a=z,S=s].
\]
The conditional relative-entropy chain rule gives
\[
\begin{aligned}
 D_{\log}(a)
 &=I_{P_{\mathrm F}^a}(Y;X\mid Z_a,S)\\
 &\quad+\mathbb E_{P_{\mathrm F}^a}D_{\mathrm{KL}}\!\left(
 \bar p_{\mathrm F}^a(\cdot\mid Z_a,S)
 \Vert q_a(\cdot\mid Z_a,S)\right).
\end{aligned}
\]
Thus the operational distortion separates two effects of compression:
predictive information absent from the policy-visible state and residual
mismatch after conditioning on that state.
This decomposition interprets the ideal objective; SeqCalib
instead uses the local activation criterion in
Eq.~\eqref{eq:conditional-similarity}. The budgeted objective trades these
effects against nominal configurable-KV storage.

%% file: sections/evaluation_details.tex
\section{Supplementary Evaluation}

\subsection{Experimental Protocol}
\label{app:evaluation-protocol}

\paragraph{Calibration and quality.}
Calibration uses FlashAttention and stores the artifact identifier, seed, and
final policy matrix for each run. The full-window check described in
Section~\ref{app:deployment} aborts a run when its cosine similarity falls
below 0.99.
RULER~\citep{ruler} uses the tasks
\url{niah_single_1}, \url{niah_single_2},
\url{niah_multikey_1}, \url{niah_multivalue}, and
\url{niah_multiquery} at \(\{4,8,16,32,64,128\}\)K. Using the official code
with seed 42, we generate 40 examples per task--context cell, or 1,200 per
configuration. LoCoMo~\citep{locomo} covers categories
1--4 from all ten conversations; a configuration enters the results only after
all 1,540 generations and all 1,540 judgments complete successfully.

\paragraph{Systems measurements.}
The Qwen3.6-27B study uses F16 K/V, FlashAttention, full GPU offload, a logical
batch size of 2048, and a physical microbatch size of 512, with automatic batch
fitting and built-in warm-up disabled. Of 128 generated tokens, 127 post-prompt
decode intervals are timed. Each successful method--context pair has three
fresh-server repetitions; throughput is pooled as total timed tokens divided by
total measured time, and memory is the median of the three device-wide
whole-invocation peaks. Separate 1K-token capacity scans vary only context
length, holding the model, policy, runtime build, GPU, and request format fixed.
We repeat the last success and first OOM three times to confirm the 1K-grid
capacity boundary.

\paragraph{Baseline controls.}
StreamingLLM uses a 4-token sink and a 90,364-token recent window. Its
resulting 90,368-token cache matches 68.95\% retention only at the 128K
reference context. Eviction begins above 90,368 tokens, so the curve reuses
matched Full-KV measurements until
independent adapted-control runs begin at 96K. Full-KV and HeadWiseKV use a
counterbalanced run order. AdaKV and HeadKV-R2 compact after prefill, whereas
StreamingLLM and adapted DuoAttention preallocate compressed caches. Only the
matched Full-KV/HeadWiseKV cohort yields systems ratios; other baselines are
reported as absolute measurements. Section~\ref{app:deployment} specifies the
StreamingLLM and DuoAttention adaptations.

%% file: sections/qwen35_9b_throughput_figure.tex
\begin{figure*}[t]
\centering
\includegraphics[width=\textwidth]{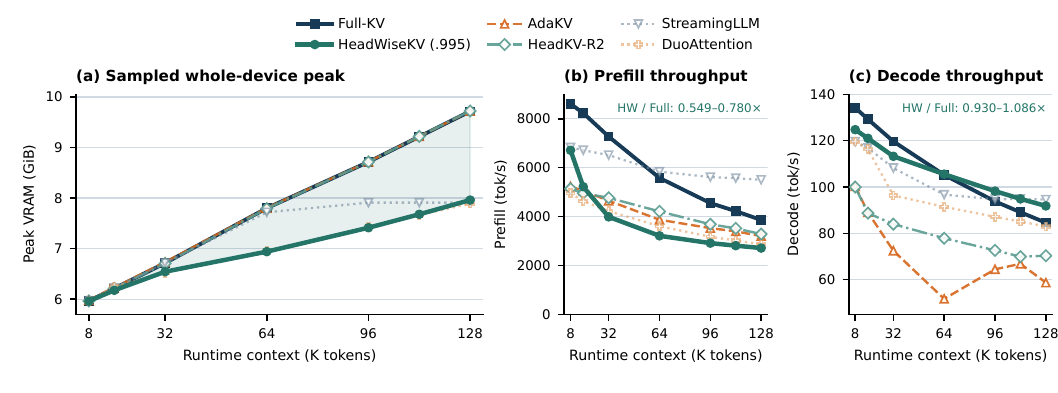}
\caption{\textbf{Single-run context sweep for Qwen3.5-9B Q4\_K\_M on one
RTX 4090 D.} The panels report sampled
whole-invocation peak VRAM, prefill throughput, and post-prompt decode
throughput at 8, 16, 32, 64, 96, 112, and 128K runtime slots. Each point is
one fresh-server invocation that generates 128 tokens, so the curves have no
error bars. HeadWiseKV uses \(\tau=.995\). StreamingLLM at 8--64K and DuoAttention
at 8--16K reserve 512 tokens for the fixed-cache \texttt{ubatch}, so these are
independently measured compressed-layout points.}
\label{fig:qwen35-9b-throughput}
\end{figure*}

%% file: sections/runtime_implementation.tex
\subsection{Runtime Realization and Deployment}
\label{app:deployment}

\paragraph{Policy and calibration.}
The runtime consumes a static SeqCalib matrix indexed by configurable
full-attention layer and physical KV head. Each entry selects either full
history or a positive local window; inference never recomputes the matrix from
the prompt, attention scores, or request history. Each reported operating point
deploys its individually calibrated matrix; the outer ordered-grid selector was
not used in these experiments.

We implement the same runtime contract for Qwen and Gemma. Each backend loads
the layer--KV-head window matrix and exposes the post-transform \(Q/K/V\)
tensors together with the conditional full-history pre-gate output required by
calibration. For reproducibility, the reported Qwen runs use the buildable
\texttt{llama.cpp} source tree
{\footnotesize\url{src_runtime/llama.cpp-9a532ae-kvmemfix}}.

Calibration and serving share the same layer--head indexing and window
semantics. At layer \(\ell\), the calibrator activates
\(\widehat a_{<\ell}\), keeps the current layer at full history, and traces its
post-transform \(Q/K/V\) tensors and conditional full-history pre-gate output.
Candidate suffixes are reconstructed from these arrays; selected entries remain
active in higher-layer traces and in the final server configuration. The
full-window replay described in Section~\ref{app:evaluation-protocol} checks
tensor layout, head mapping, and attention scaling against
Eq.~\eqref{eq:conditional-attention}. Differences in kernel precision or
accumulation order can make candidates near \(\tau\) path-sensitive.

\paragraph{Physical cache realization.}
The runtime assumes known full-attention layers, a mapping from physical KV
heads to GQA query-head groups, and a window codebook representable as cache
groups.
\newpage
Whereas the baseline manages residency by layer, HeadWiseKV partitions
configurable full-attention storage into active
\((\mathrm{window},\mathrm{KV\ head})\) groups. Layers using the same window at
a given KV-head index reuse the group type and capacity while retaining
disjoint, layer-specific K/V entries; unassigned layers and heads allocate no
entries in that group. Native local and recurrent layers retain their original
cache path.

Each configured KV head reads and writes only its assigned K/V slice. A full
entry uses a full-length single-head group, while a finite entry uses the
corresponding local-window group. Each GQA query-head group follows its mapped
KV head. Per-head attention outputs are concatenated before the unchanged
gating and output projection, so the prompt is not globally truncated. Full-KV
and HeadWiseKV use identical model weights and differ only in cache layout and
graph routing. Porting requires per-head K/V access and preservation of the GQA
mapping.

\paragraph{Accounting and adapted controls.}
Normalized retention counts token--head entries only in the configurable
full-attention layers; weights, workspaces, allocator effects, and
non-attention state remain outside this accounting. Peak VRAM and confirmed
context capacity are therefore measured on the complete runtime invocation.
The adapted StreamingLLM control preallocates its sink--recent cache and retains
the original Qwen RoPE positions rather than applying the official
cache-relative position shift. The DuoAttention control uses a fixed binary
retrieval/streaming-head projection rather than the official learned gate.
These adaptations define the two policy-level controls used in the reported
systems traces.

%% file: sections/longmemeval_results.tex
\subsection{Qwen3.5-9B Quality and Systems Results}
\label{app:longmemeval}

\paragraph{LongMemEval-S quality and memory.}
Table~\ref{tab:longmemeval-qwen35-9b} compares Qwen3.5-9B Q4\_K\_M on the
cleaned 500-question LongMemEval-S set~\citep{longmemeval}. All configurations
completed the 500 questions with thinking disabled, temperature 0, and seed 42.
Inputs exceeding 122,880 tokens after template application retain equal-sized
content-token prefixes and suffixes; outputs are capped at 8,192 tokens. A fixed
DeepSeek V4 Pro binary judge scored every answer. We audited all labels under a
lenient contains-the-correct-answer rule and changed only unambiguous false
positives or false negatives. Reviewed ACC is the accuracy computed from these
audited labels. Here, \(R^{Q}_{128\mathrm K}\) is the average
configurable query-head KV retention at the 131,072-token reference context,
while \(\tau\) denotes HeadWiseKV's calibration-fidelity threshold.

\begin{table}[t]
\centering
\caption{\textbf{LongMemEval-S accuracy and peak memory for Qwen3.5-9B
Q4\_K\_M.} Reviewed ACC is the post-audit score from the shared 500-question
protocol. Peak VRAM comes from a separate standardized single-RTX-4090-D
workload with
a 130,943-token prompt and 128 generated tokens in a 131,072-token runtime
slot; each value is the median of three fresh-server, whole-invocation
physical-device peaks. Arrows indicate the preferred direction. Full-KV is
the uncompressed reference.}
\label{tab:longmemeval-qwen35-9b}
\footnotesize
\setlength{\tabcolsep}{0.9pt}
\renewcommand{\arraystretch}{1.06}
\begin{tabular}{@{}lrrr@{}}
\toprule
Method &
\shortstack{\(R^{Q}_{128\mathrm K}\)\\(\%) \(\downarrow\)} &
\shortstack{Reviewed\\ACC (\%) \(\uparrow\)} &
\shortstack{Peak VRAM\\(MiB) \(\downarrow\)} \\
\midrule
Full-KV & 100.00 & 70.2 & 9947 \\
\midrule
AdaKV~\citep{adakv} & 53.71 & 54.2 & 9951 \\
HeadKV-R2~\citep{headkv} & 53.71 & 43.6 & 9955 \\
StreamingLLM~\citep{streamingllm} & 53.71 & 12.8 & 8099 \\
DuoAttention~\citep{duoattention} & 53.71 & 14.8 & 8081 \\
\midrule
HeadWiseKV (\(\tau=.999\)) & 80.27 & 70.6 & 9225 \\
HeadWiseKV (\(\tau=.995\)) & 53.71 & 47.0 & 8151 \\
HeadWiseKV (\(\tau=.990\)) & 26.76 & 25.6 & 7053 \\
\bottomrule
\end{tabular}
\end{table}

At the shared \(R^{Q}_{128\mathrm K}=53.71\%\) point, AdaKV records 54.2\%
Reviewed ACC, followed by HeadWiseKV at 47.0\% and HeadKV-R2 at 43.6\%.
HeadWiseKV at
\(\tau=.999\) reaches 70.6\% while retaining 80.27\% of configurable KV,
compared with 70.2\% for Full-KV.
The 0.4-point difference corresponds to two questions and does not establish a
ranking between those configurations under this audit.

AdaKV and HeadKV-R2 compact only after full-context prefill, so their
whole-invocation peaks remain close to Full-KV even though their decode-steady
peaks are 8109 and 7529 MiB, respectively. HeadKV-R2's 53.71\% entry is its
configured query-head budget equivalent; its unioned physical payload retention
is 39.44\%.

\paragraph{Single-GPU context sweep.}
At 128K, HeadWiseKV (\(\tau=.995\)) uses 7.96 GiB versus Full-KV's 9.71 GiB
sampled peak, while its observed prefill/decode throughputs are 2714/91.8
tok/s versus 3862/84.6 tok/s. Across all seven contexts, HeadWiseKV attains
0.549--0.780 of Full-KV prefill throughput and 0.930--1.086 of Full-KV decode
throughput. StreamingLLM's fixed physical allocation plateaus near 7.91 GiB
from 96K onward and reaches 5503/94.5 prefill/decode tok/s at 128K. These curves
isolate systems behavior, while Table~\ref{tab:longmemeval-qwen35-9b} provides
the separate quality comparison and repeated 128K memory campaign.

%% file: sections/limitations.tex
\subsection{Limitations}
\label{app:limitations}

HeadWiseKV fixes a prompt-independent window matrix, so it cannot recover
evicted evidence or adapt to requests with atypical long-range dependencies.
SeqCalib selects profiles using in-sample reconstruction cosine. Profiles must
be revalidated when the model, cache format, or target context changes, and new
workloads require task-specific validation; material shifts may also require
recalibration. Because selection is layerwise, the resulting profile need not
minimize a global allocation objective. Quality conclusions cover only the
reported tasks and operating
points, while aggregate scores do not support paired per-example uncertainty
analysis. Memory and throughput depend on the tested hardware, backend, batch
settings, quantization, and context length, and token--head retention does not
by itself determine allocated bytes or peak VRAM. Because the baseline controls
use different cache lifecycles and the adaptations described in
Section~\ref{app:deployment}, their absolute measurements characterize only
these implementations.